\documentclass[amssymb,10pt]{article}
\usepackage[dvipdfm]{graphicx}
\usepackage{amsthm}
\usepackage[usenames]{color}
\usepackage{xspace,colortbl}
\usepackage{latexsym,url,amscd,amssymb}
\usepackage{amsmath}
\usepackage{graphics}
\usepackage{epstopdf}
\usepackage{hyperref}
\usepackage{cases}
\usepackage{caption}
\usepackage{algorithm}
\usepackage{algpseudocode}
\usepackage{listings}
\usepackage{array}
\usepackage{multirow}
\usepackage{makecell}
\usepackage{framed}

\def \bR {\mathbb{R}}

\newtheorem{lemma}{Lemma}[section]
\newtheorem{definition}{Definition}[section]

\newtheorem{theorem}{Theorem}[section]

\newtheorem{proposition}{Proposition}[section]

\newcommand{\tabincell}[2]

\title {Just Least Squares:  Binary  Compressive Sampling with Low Generative Intrinsic Dimension}

\author{
Yuling Jiao\thanks{School of Mathematics and Statistics and Hubei Key Laboratory of Computational Science, Wuhan University, Wuhan 430072, China (Email: yulingjiaomath@whu.edu.cn).}
\and
Dingwei Li \thanks{School of Mathematics and Statistics, Wuhan University, Wuhan 430072, China (Email:dueli@whu.edu.cn). }
\and
Min Liu\thanks{School of Mathematics and Statistics, Wuhan University, Wuhan 430072, China (Email: mliuf@whu.edu.cn).}
\and
Xiliang Lu\thanks{School of Mathematics and Statistics and Hubei Key Laboratory of Computational Science, Wuhan University, Wuhan 430072, China (Email: xllv.math@whu.edu.cn).}
\and
Yuanyuan Yang\thanks{School of Mathematics and Statistics, Wuhan University, Wuhan 430072, China (Email: yuanyuanyang@whu.edu.cn).}
}

\date{ }
\begin{document}
\maketitle

\begin{abstract}
In this paper, we consider recovering $n$ dimensional signals from $m$ binary measurements corrupted by
noises and sign flips under the assumption that the target signals have low generative intrinsic
dimension, i.e., the target signals can be approximately generated via an $L$-Lipschitz generator
$G: \mathbb{R}^k\rightarrow\mathbb{R}^{n}, k\ll n$.
Although the binary measurements model is highly nonlinear,
we propose a least square decoder and prove that, up to a constant $c$, with high probability,
the least square decoder achieves a sharp estimation error $\mathcal{O} (\sqrt{\frac{k\log (Ln)}{m}})$ as long as
$m\geq  \mathcal{O}( k\log (Ln))$. Extensive numerical simulations and comparisons with state-of-the-art methods demonstrated
the least square decoder is robust to noise and sign flips, as indicated by our theory.
By constructing a ReLU network with  properly chosen depth and width,
we verify  the (approximately) deep generative prior, which is  of independent interest.
\end{abstract}

\setcounter{equation}{0}
\section{Introduction}
Compressive sensing is a powerful signal acquisition approach  with  which one  can  recover signals beyond bandlimitedness from noisy under-determined measurements whose number  is closer to the intrinsic complexity of the target signals
than the Nyquist rate \cite{CandesRombergTao:2006,Donoho:2006,FazelCandesRcht:2008,FoucartRauhut:2013}.
Quantization that transforms the infinite-precision measurements into discrete ones is necessary for storage and transmission \cite{Sayood:2017}.
The binary  quantizer, an extreme case of scalar quantization,
that   codes the measurements  into binary values with  a single bit has been introduced into compressed sensing
\cite{BoufounosBaraniuk:2008}. The 1-bit compressed sensing (1-bit CS) has drawn much attention because of its low cost in hardware mentation and storage  and its robustness in the low signal-to-noise ratio scenario  \cite{Laska:2012}.
\subsection{Related  work}
A lot of  efforts  have been devoted to studying the theoretical and computational issues  in  the 1-bit CS under the sparsity assumption, i.e., $\|x^*\|_0\leq s \ll m$.
 Support recovery  can be achieved  in both  noiseless and noisy setting provided that
$m > \mathcal{O}(s\log n)$ \cite{GopiJain:2013,JacquesDegraux:2013,PlanVershynincpam:2013,BauptBaraniuk:2011,JacquesLaska:2013,GuptaNowakRech:2010,BauptBaraniuk:2011,PlanVershynin:2013,ZhangYiJin:2014,Ahsen:2019}. Greedy methods \cite{LiuGongXu:2016,Boufounos:2009,JacquesLaska:2013} and first order methods \cite{BoufounosBaraniuk:2008, LaskaWenYinBaraniuk:2011,YanYangOsher:2012,DaiShenXuZhang:2016}  are developed  to minimize the sparsity promoting nonconvex objective function caused by    the unit sphere constraint or the  nonconvex regularizers. Convex  relaxation models are also proposed \cite{ZhangYiJin:2014,PlanVershynin:2013,PlanVershynincpam:2013, ZymnisBoydCandes:2010,PlanVershynin:2017,Vershynin:2015,HuangShiYan:2015}  to address the nonconvex optimization  problem.
Using least squares to estimate parameters in the scenario of   model  misspecification    goes back to
\cite{Brillinger:2012}, and see also  \cite{LiDuan:1989} and the references therein  for related  development  in the setting $m \gg n$.

Recently, with this idea,  \cite{PlanVershynin:2016,Neykov:2016,HuangJiao:2018,ding2020robust} proposed least square with $\ell_1$/$\ell_0$ regularized or generalized lasso to estimate parameters from   general under-determined  nonlinear measurements.
In addition to the sparse structure  of the  signals/images under certain linear transform \cite{Mallat:2008},  the natural signals/images data have been verified having low intrinsic dimension, i.e., they can be represented  by a generator  $G$,  such as pretrained  neural network, that maps from $\mathbb{R}^k$ to $\mathbb{R}^{n}$ with $k\ll n$. Such a  $G$ can be obtained via GAN    \cite{goodfellow14},  VAE \cite{kingma14} or flow based method \cite {rezende2015variational}. In these models, the generative part learns a mapping from a low dimensional representation space $z\in \mathbb{R}^k$ to the high dimensional sample space $G(z)\in \mathbb{R}^n$. While training, this mapping is encouraged to produce vectors that resemble the vectors in the training dataset. With this generative prior, several tasks have been studied such as image restoration  \cite{ulyanov2018deep}, phase retrieval \cite{hand2018phase} and  compressed sensing \cite{wu2019deep,bora2017compressed,huang2018provably,liu2020information} and  nonlinear single index models under certain measurement and noise models  \cite{wei2019statistical,liu2020generalized}.

 In \cite{bora2017compressed}, the authors propose the least squares estimator (\ref{ls1})-(\ref{ls2}) to recover signals in standard compressed sensing  with generative prior and prove  sharp sample complexities \cite{liu2020information}. Surprisingly, the sharp  sample complexity for the squares decoder (\ref{ls1})-(\ref{ls2}) can be derived in this paper even if the measurements are highly quantized and corrupted by noise and sign flips.
Very recently, under generative prior, \cite{liu2020sample} and \cite{qiu2020robust} derived     sample complexity results for 1-bit CS.
The sample complexity obtained  in  \cite{liu2020sample} is $O(k\log L)$ under the assumption that the generator $G$ is $L$- Lipschitz continuous and  the rows of $A$   are   i.i.d.   sampled from   $\mathcal{N}(\textbf{0},\mathbf{I})$.  However, the estimator proposed in  \cite{liu2020sample},  $\hat{x} = G(\hat{z})$ with $\hat{z} \in \{z: y=\mathrm{sign}(AG(z))\}$,  is quite different from our least squares decoder and the analysis technique used there are also not applicable to our decoder.
 \cite{qiu2020robust} proposed unconstrained empirical
risk minimization to recovery in 1-bit CS and  derived  the sample complexity to be  $\mathcal{O}(kd\log n)$ for $d$-layer ReLU network $G$ via assuming  the rows of $A$ are i.i.d. sampled from   subexponential distributions. However, the 1-bit CS model considered in \cite{qiu2020robust} is without sign flips and require additional   quantization threshold before sampling  to measure and the empirical
risk minimization decoder used there  is also different from our least squares  (\ref{ls1})-(\ref{ls2}). The results  in \cite{liu2020generalized}
 can be applied to  1-bit CS model,  however,  it  requires that the target signals are exact contained  in the range of the generator. In contrast, we only need a more realistic assumption that the target signals can be approximated by  a generator.
\subsection{Notation and Setup}\label{setting}
We use $[n]$  to denote the set $\{1,...,n\}$, use $A_i\in \mathbb{R}^{m\times1}, i \in [n]$ and $a_j \in \mathbb{R}^{n\times 1},j\in [m]$ to denote the $i$th column and $j$th row  of $A$, respectively.
The multivariate normal distribution is denoted by $\mathcal{N}(\textbf{0},\Sigma)$ with a symmetric and positive definite matrix $\Sigma$.
Let $\|x\|_{\Sigma} = (x^{t}\Sigma x)^{\frac{1}{2}}$, and $\|x\|_p = (\sum_{i=1}^{n}|x_{i}|^p)^{1/p}, p\in [1,\infty)$  be the $\ell_p$-norm of  $x$. Without causing confusion, $\|\cdot\|$ defaults to $\|\cdot\|_2$.
Sign function $\textrm{sign}(\cdot)$ is defined componentwise as $\textrm{sign}(z) =1$ for $z \geq 0$ and $\textrm{sign}(z) = -1$ for $z<0$. We use $\odot $ to denote the Hadamard product.
For any set $B$, $|B|$ is defined as the number of elements contained in $B$.

Following \cite{PlanVershynincpam:2013,HuangJiao:2018},  we consider the following 1-bit CS model
\begin{equation}\label{setup}
 y = \eta \odot\textrm{sign} (A x^* + \epsilon),
\end{equation}
where $y\in \mathbb{R}^{m}$ are the binary  measurements, $x^{*}\in \mathbb{R}^{n}$ is an unknown signal.
The measurement matrix $A \in \mathbb{R}^{m\times n}$ is a random  matrix whose rows $a_i, i \in [m]$  are   i.i.d.  random vectors sampled from   $\mathcal{N}(\textbf{0}, \Sigma)$ with  an unknown
covariance matrix $\Sigma $, $\eta\in \mathbb{R}^{m}$ is a random vector modeling the sign flips of $y$ whose  coordinates $\eta_is$ are i.i.d.  satisfying
 $\mathbb{P}[\eta_i = 1] = 1- \mathbb{P}[\eta_i = -1] = q \neq \frac{1}{2},$ and $\epsilon \in \mathbb{R}^{m}$ is a random vector sampled from  $\mathcal{N}(\textbf{0},\sigma^2\textbf{I}_m)$ with an
unknown noise level $\sigma$ modeling errors before quantization.   We assume  $\eta_i, \epsilon_i$ and $a_i$ are independent.

Model \eqref{setup} is  unidentifiable under positive scaling, the best one can do is to recover $x^*$ up to a constant, $c=(2q-1)\sqrt{\frac{2}{\pi(\sigma^2+1)}}$ which has been proved in \cite{HuangJiao:2018}. Without loss of generality
we may assume $\| x^*\|_{\Sigma} = 1$.

Let $\ell$-dimensional unit sphere and a ball in the $\ell^p$ norm to be
$$\mathcal{S}_{p}^{\ell-1} = \{x\in\mathbb{R}^{\ell}: \|x\|_p=1\}, \quad \mathcal{B}_p^{\ell}(r)=\{z\in\mathbb{R}^{\ell}: \|z\|_p\leq r\}.$$
For an $L$-Lipschitz generator $G: \mathbb{R}^k\rightarrow\mathbb{R}^{n}$, denote by
$$\mathcal{G}_{k,\tau,p}(r) = \{x\in\mathbb{R}^{n}: \exists  z\in \mathcal{B}_{2}^{k}(r), \ \ s. t. \ \ \|cx-G(z)\|_p \leq \tau\},$$
where signals that can generated by $G$ with tolerance $\tau$. When $p=2$, we denote $\mathcal{G}_{k,\tau,2}(r)$ by $\mathcal{G}_{k,\tau}(r)$ for simplicity. The target signal $x^*$ is assumed with
low generative intrinsic dimension, i.e., $x^*\in \mathcal{G}_{k,\tau,p}(r)$ for some $p$ and $r$.

\subsection{Contributions}
It is a  challenging task  to decode from nonlinear, noisy, sign-flipped and under determined ($m\ll n$)  binary  measurements.
For a given Lipschtiz generator $G$, we use $\hat{x} =G(\hat{z})$ to estimate $x^*$  in  the 1-bit CS model \eqref{setup} via exploring the intrinsic low dimensional structure of the target signals,
 where the latent code $\hat{z}$ is solved by the least square problem \eqref{ls2}.
 \begin{itemize}
 \item[(1)] We prove  that, with high probability  the estimation error
   $\|\hat{x}- c x^*\| \leq \mathcal{O} (\sqrt{\frac{k\log (Ln)}{m}})$
   is  sharp provided that   the sample complexity satisfies
   $m\geq  \mathcal{O}( k  \log (Ln))$, if the target signal $x^*$ can  be approximated well by generate $G$.
    \item[(2)] By constructing a ReLU network with  properly chosen depth and width,
we verify  the desired approximation in (1) holds if the target signals have low intrinsic dimensions.
 \item[(3)] Extensive numerical simulations    and comparisons with state-of-the-art methods  show that
the proposed least square decoder is the robust to noise and sign flips,  as demonstrated  by our theory.
\end{itemize}

The rest of the paper is organized as follows. In Section 2  we  consider the  least squares decoder  and prove several  bounds on   $\|\hat{x} -cx^* \|$.
In Section 3 we conduct numerical simulation  and compare  with  existing  state-of-the-art 1-bit CS  methods.
We conclude  in Section 4.
\section{Analysis of the Least Square Decoder}
We first propose the least square decoder in details. Consider the following least square problem for the latent code $z$:
 \begin{equation}\label{ls2}
 \hat{z} \in \arg \min_{z\in \mathcal{B}_2^{k}(r)} \frac{1}{2m}\|y - AG(z)\|^2.
  \end{equation}
Then for a given $L$-Lipschtiz generator $G$, the signal is approximated by
   \begin{equation}\label{ls1}
 \hat{x} =G(\hat{z}).
 \end{equation}
In this section, we will prove under proper assumption on generator and sample complexity, the error between the decoder $\hat{x}$ and the underlying signal $x^*$ can be estimated, i.e., Theorem \ref{thel1} and \ref{theorem2}. Moreover we also provide the construction of a ReLU network such that the approximation to the target signals are satisfied, see Theorem \ref{thapp}.
\begin{theorem}\label{thel1}
Given a Lipschitz generator satisfying  $G(\mathcal{B}_2^k(1))\subset \mathcal{B}_1^n(1)$.
Assume the 1-bit CS model \eqref{setup} holds with  $x^* \in \mathcal{G}_{k,\tau,1}(1)$, and  $m \geq \mathcal{O} \left(\max \{\log n, k\log {\frac{L}{\tau}}\}\right)$, then with  probability at least $1-O(\frac{1}{n^2})-e^{-O(m/k)}$, the least squares decoder defined in (\ref{ls2})-(\ref{ls1}) (for $r=1$) satisfies
\begin{equation*}
\|\hat{x}-c{x^*}\|\leq \mathcal{O}\left(\sqrt{\tau} + \left(\frac{\log n}{m}\right)^{1/4}\right).
\end{equation*}
\end{theorem}

To prove Theorem \ref{thel1}, we need some technical Lemmas. Firstly we introduce the concept of S-REC with some minor changes and $\epsilon-$net, which is defined in \cite{bora2017compressed}.
\begin{definition}\cite{bora2017compressed}.
Let $S\subseteq\mathbb{R}^n$ and two positive parameters $\gamma>0, \delta>0$. The matrix $A\in \mathbb
{R}^{m\times n}$ is said to satisfy the S-REC$(S, \gamma, \delta)$, if  $\forall x_1, x_2\in S$,
\begin{equation}
\frac{1}{m}\|A(x_1-x_2)\|^2\geq\gamma\|x_1 - x_2\|^2 - \delta.
\end{equation}
\end{definition}
\begin{definition}
Let $N \subseteq S\subseteq\mathbb{R}^n$ and $\epsilon > 0$. We say that $N$ is an $\epsilon-$net of $S$, if
$\forall s \in S$, there exist an $\tilde{s} \in N$ such that $\|s-\tilde{s}\|\leq \epsilon$.
\end{definition}

\begin{lemma}\label{cov}\cite{boucheron2013concentration}
$\forall \epsilon > 0$, there exists an $\epsilon-$net $N_{\epsilon}$ of $\mathcal{B}^k_2(r)$ with finite many points in $N_\epsilon$, such that
 $$\log|N_\epsilon|\leq k\log(\frac{4r}{\epsilon}).$$
\end{lemma}
The proof follows directly from the standard volume arguments, see \cite{boucheron2013concentration}.

\begin{lemma}\label{A_delta}
Let $G: \mathbb{R}^k\rightarrow\mathbb{R}^n$ be an $L$-Lipschitz function.
If $N$ is a $\frac{\delta}{L}-$net on $\mathcal{B}^k_2(r)$, then, $G(N)$ is a $\delta-$net on $G(\mathcal{B}^k_2(r))$, i.e.,
\begin{equation}\label{eq:1}
\forall z\in \mathcal{B}^k_2(r), \quad \exists z_1 \in N, \; s.t. \;\|G(z)-G(z_1)\|\leq \delta.
\end{equation}
Furthermore, let
$A\in\mathbb{R}^{m\times n}$ be a random matrix and the rows are i.i.d. random vectors sampled from the multivariate normal distribution $\mathcal{N}(\mathbf{0},\Sigma)$, then,
\begin{equation}\label{eq:2}
\frac{1}{\sqrt{m}}\|AG(z)-AG(z_1)\|\leq\mathcal{O}(\delta)
\end{equation}
holds with probability $1-e^{- O(m)}$ as long as $m = \mathcal{O}\left(k\log\frac{L}{\delta}\right)$.
\end{lemma}
\begin{proof}
Let $N$ be $\frac{\delta}{L}-$ net on $\mathcal{B}_2^k(r)$ satisfying $$\log|N|\leq k\log(\frac{4Lr}{\delta}).$$ Since G is $L$-Lipschitz function, then by definition  we can check that  $ G(N)$ is $\delta-$net on $G(\mathcal{B}_2^k(r))$.

For fixed $\delta >0$, let $N_i$ be a $\frac{\delta_i}{L}-$net on $\mathcal{B}^k_2(r)$ satisfying  $\log|N_i|\leq k\log\frac{4Lr}{\delta_i}$  with $\delta_i = \frac{\delta}{2^i}$, and  $$N = N_0\subset N_1\subset\ldots\subset N_l,$$  with $2^l>\sqrt{n}.$  \\
$\forall x \in G(\mathcal{B}^k_2(r))$, $\exists x_i\in G(N_i)$, such that $$\|x - x_l\|\leq\frac{\delta}{2^l} \text{ and } \|x_{i+1} - x_i\|\leq\frac{\delta}{2^i}, i= 1, \ldots, l-1.$$\\
By triangle inequality we get that,
\begin{equation}\label{eqlem1}
\begin{array}{l}
\frac{1}{\sqrt{m}}\|Ax-Ax_0\|\\
= \|\frac{1}{\sqrt{m}}A\sum_{i=0}^{l-1}(x_{i+1}-x_i) + \frac{1}{\sqrt{m}}A(x-x_l)\|\\
\leq\sum_{i=0}^{l-1}\frac{1}{\sqrt{m}}\|A(x_{i+1}-x_i)\|+\|\frac{1}{\sqrt{m}}A(x-x_l)\|.
\end{array}
\end{equation}
By construction, the last term
\begin{equation}\label{eqlem2}
\begin{array}{l}
\|\frac{1}{\sqrt{m}}A(x-x_l)\|\leq(2+\sqrt{\frac{n}{m}})\frac{\delta}{2^l}=\mathcal{O}({\delta}).
\end{array}
\end{equation}
Let $\widetilde{A} = A\Sigma^{-1/2}$,
by Lemma 1.3 in \cite{vempala2005random}, with probability at least $1-\exp(-\mathcal{O}(\epsilon_i^2m))$, the following holds
\begin{equation*}
\|\frac{1}{\sqrt{m}}\widetilde{A}(x_{i+1}-x_i)\|^2\leq(1+\epsilon_i)\|x_{i+1}-x_i\|^2,
\end{equation*}
equivalently,
\begin{equation*}
\begin{array}{l}
\|\frac{1}{\sqrt{m}}A\Sigma^{-\frac{1}{2}}\Sigma^{\frac{1}{2}}(x_{i+1}-x_i)\|^2\\
\leq(1+\epsilon_i)\|\Sigma^{\frac{1}{2}}\|^2\|x_{i+1}-x_i\|^2,
\end{array}
\end{equation*}
i.e.,
\begin{equation}\label{ineq_A_x_i}
\|\frac{1}{\sqrt{m}}A(x_{i+1}-x_i)\|
\leq(1+\frac{\epsilon_i}{2})\|\Sigma^{\frac{1}{2}}\|\|x_{i+1}-x_i\|,
\end{equation}
the last inequality is derived from  $\sqrt{1+\epsilon_i}\leq 1+\frac{\epsilon_i}{2}, \text{ }\epsilon_i\in(0,1)$.
Set $\epsilon_i^2 =\epsilon +\frac{ik}{m}$, and use union bound and \eqref{ineq_A_x_i}, we have 
$\forall i\in[l]$,
\begin{equation}\label{eqlem3}
\|\frac{1}{\sqrt{m}}A(x_{i+1}-x_i)\|
\leq(1+\frac{\epsilon_i}{2})\|\Sigma^{1/2}\|\|x_{i+1}-x_i\|,
\end{equation}
with probability at least $1-\exp{(-\mathcal{O}(\epsilon m))}$.
Then, it follow from \eqref{eqlem1}, \eqref{eqlem2} and \eqref{eqlem3} that,
\begin{equation*}
\begin{array}{l}
\frac{1}{\sqrt{m}}\|Ax-Ax_0\|\\
\leq\|\frac{1}{\sqrt{m}}A\sum_{i=0}^{l-1}(x_{i+1}-x_i)\| + \mathcal{O}(\delta)\\
\leq\sum_{i=0}^{l-1}(1+\frac{\epsilon_i}{2})(\sigma_{max}(\Sigma))^{\frac{1}{4}}\frac{\delta}{2^i} + \mathcal{O}(\delta)\\
\leq \delta(\sigma_{max}(\Sigma))^{\frac{1}{4}}\sum_{i=0}^{l-1}\frac{\sqrt{\epsilon}}{2^{i+1}}(1+\frac{ik}{m\epsilon})+\mathcal{O}(\delta)\\
=\mathcal{O}(\delta).
\end{array}
\end{equation*}
\end{proof}

\begin{lemma}\label{A_REC}
Let $G: \mathbb{R}^k\rightarrow\mathbb{R}^n$ be $L$-Lipschitz generator, $S = G(\mathcal{B}^k_2(r))$,
and $A\in\mathbb{R}^{m\times n}$ be a random matrix and the rows are i.i.d. random vectors sampled from the multivariate normal distribution $\mathcal{N}(\mathbf{0},\Sigma)$. if $m = \mathcal{O}\left(k\log\frac{Lr}{\delta}\right),$ $A$ satisfy the S-REC$(S, \frac{1}{2}\sqrt{\sigma_{min}(\Sigma)}, O(\delta))$, with probability $1 - e^{-O(m/k)}$.
\end{lemma}

\begin{proof}
We construct a $\frac{\delta}{L}-$net on $\mathcal{B}_2^k(r)$, which is denoted as  $N$  and satisfy $\log|N|\leq k\log(\frac{4Lr}{\delta}).$ Since G is $L$-Lipschitz function, then by Lemma \ref{A_delta},   $ G(N)$ is $\delta-$net on $G(\mathcal{B}_2^k(r))$, i.e.,\\
$\forall z, z'\in \mathcal{B}_2^k(r), \exists z_1, z_2\in N$ s.t.
\begin{equation}\label{lem2eq1}
\begin{aligned}
\|z-z_1\|\leq\frac{\delta}{L}, \ \  \|G(z)-G(z_1)\|\leq\delta,\\
\|z'-z_2\|\leq\frac{\delta}{L}, \ \  \|G(z')-G(z_2)\|\leq\delta.
\end{aligned}
\end{equation}

 By triangle inequality, Lemma \ref{A_delta} and \eqref{lem2eq1}, we get
\begin{equation}\label{ineq_G}
\begin{array}{ll}
\|G(z)-G(z')\|
&\leq\|G(z)-G(z_1)\|+\|G(z_1)-G(z_2)\| +\|G(z_2)-G(z')\| \\[1.5ex]
& \leq 2\delta + \|G(z_1)-G(z_2)\|
\end{array}
\end{equation}
and
\begin{equation}\label{A_z1_z2}
\begin{array}{ll}
\frac{1}{\sqrt{m}}\|AG(z_1)-AG(z_2)\|
&\leq\frac{1}{\sqrt{m}}\left(\|AG(z_1)-AG(z)\|+\|AG(z)-AG(z')\| +\|AG(z')-AG(z_2)\|\right)\\[1.5ex]
&\leq \mathcal{O}(\delta) + \frac{1}{\sqrt{m}}\|AG(z)-AG(z')\|.
\end{array}
\end{equation}
Recall $N$ is a $\frac{\delta}{L}-$net on $\mathcal{B}_2^k(r)$, consider
$$G(N) = \{G(z): z\in N\}, \quad T= \Sigma^{\frac{1}{2}}G(N) = \{t: t=\Sigma^{\frac{1}{2}} G(z), z\in N\}, $$
then $|T| \leq  |G(N)| \leq |N|\leq (\frac{4Lr}{\delta})^k$. Similar as Lemma \ref{A_delta}, let $\widetilde{A}=A\Sigma^{-\frac{1}{2}}$, then the rows of $\widetilde{A}$ are i.i.d standard Gaussian vectors. By the Johnson-Lindenstrauss Lemma, the projection $F: \mathbb{R}^n\rightarrow \mathbb{R}^m$ with $F(t) = \frac{1}{\sqrt{m}}A\Sigma^{-\frac{1}{2}}t$ preserves distances in the sense that, given any $\epsilon\in (0,1)$, with probability at least $1-e^{-\mathcal{O}(\epsilon^2 m/k)}$, for all $t_1, t_2\in T$,
\begin{equation*}
(1-\epsilon)\|t_1 - t_2\|\leq\|F(t_1)-F(t_2)\|\leq(1+\epsilon)\|t_1-t_2\|
\end{equation*}
provided that $m \geq \mathcal{O}(\frac{k}{\epsilon^2}\log\frac{Lr}{\delta}).$
We may choose $\epsilon = 0.5$ and hence
\begin{equation}\label{J-L_A}
\frac{1}{\sqrt{m}}\|AG(z_1)-AG(z_2)\|
\geq 0.5\|\Sigma^{\frac{1}{2}}(G(z_1)-G(z_2))\| \geq 0.5 \sqrt{\sigma_{min}(\Sigma)} \|G(z_1)-G(z_2)\|,
\end{equation}
holds with probability at least $1-e^{-\mathcal{O}(m/k)}.$
 It follows from   \eqref{ineq_G}-\eqref{J-L_A} that
\begin{equation*}
\begin{array}{ll}
\frac{1}{\sqrt{m}}\|AG(z)-AG(z')\|
&\geq  \frac{1}{\sqrt{m}}\|AG(z_1)-AG(z_2)\| - \mathcal{O}(\delta) \\ [1.5ex]
 &\geq 0.5 \sqrt{\sigma_{min}(\Sigma)} \|G(z_1)-G(z_2)\| - \mathcal{O}(\delta) \\ [1.5ex]
&\geq  0.5 \sqrt{\sigma_{min}(\Sigma)} \|G(z)-G(z')\| - \mathcal{O}(\delta).
\end{array}
\end{equation*}
The above inequality implies that  $A$ satisfy the S-REC$(G(\mathcal{B}^k_2(r)),0.5\sqrt{\sigma_{min}(\Sigma)}, O(\delta))$ with probability at least $1-e^{-\Omega(m/k)}$, for $m \geq \mathcal{O}(k\log\frac{Lr}{\delta}).$
\end{proof}

Next Lemma shows that least square decoder can be good in the subgaussian setting.
\begin{lemma}\label{linf} \cite{HuangJiao:2018}
Let $A\in\mathbb{R}^{m\times n}$, whose rows  $a_i\in\mathbb{R}^n$, are independent subgaussian vectors with mean $\mathbf{0}$ and covariance matrix $\Sigma$. If  $m\geq \mathcal{O}(\log n)$,  then
\begin{equation}
\left\|\sum_{i=1}^{m}\left(\mathbb{E}\left[a_{i} y_{i}\right]-a_{i} y_{i}\right) / m\right\|_{\infty} \leq  \mathcal{O}(\sqrt{\frac{\log n}{m}})
\end{equation}
 holds with probability at least $1-\frac{2}{n^3}$,
and
\begin{equation}
\left\|A^T A / m-\Sigma\right\|_{\infty} \leq  \mathcal{O}(\sqrt{\frac{\log n}{m}})
\end{equation}
holds with probability at least $1-\frac{1}{n^2}$, where $\|\Psi\|_\infty$ is the maximum pointwise absolute value of $\Psi$.
\end{lemma}

Now we are ready to prove Theorem \ref{thel1}.

\begin{proof}
Recall that
$$y = \eta\odot sign(Ax^*+\epsilon)$$ and
\begin{equation}
 \widehat{z} = \arg\min_{z\in \mathcal{B}_2^k(1)}\frac{1}{2m}\|y - AG(z)\|^2.
\end{equation}
Our goal is to  bound  $\|G(\widehat{z})-\widetilde{x^*}\|_2$ with $\widetilde{x^*}=cx^*$.
By triangle inequality,
\begin{equation*}
\begin{aligned}
\|G(\widehat{z})-\widetilde{x^*}\|
&=\|G(\widehat{z})- G(\overline{z}) + G(\overline{z})-\widetilde{x^*}\|\\
&\leq \|G(\widehat{z})- G(\overline{z})\| + \|G(\overline{z})-\widetilde{x^*}\|,
\end{aligned}
\end{equation*}
where $\overline{z}\in \mathcal{B}_{2}^{k}(1)$ is chosen such that  $\|G(\overline{z})-\widetilde{x^*}\|_1 \leq \tau$ by the assumption $x^*\in \mathcal{G}_{k,\tau,1}(1)$, 
we have
\begin{equation}\label{the1}
\|G(\widehat{z})-\widetilde{x^*}\| \leq \|G(\widehat{z})- G(\overline{z})\|+ \tau.
\end{equation}
From the definition of $\widehat{z}$ we have
$$\|AG(\widehat{z}) - y\|^2 \leq \|AG(\overline{z}) - y\|^2.$$
Direct computation shows that
\begin{equation*}
\begin{aligned}
0 &\geq \|AG(\widehat{z}) - y\|^2  -\|AG(\overline{z}) - y\|^2 = \|AG(\widehat{z}) - AG(\overline{z})+ AG(\overline{z})- y\|^2 -\|AG(\overline{z}) - y\|^2 \\
& = \|AG(\widehat{z}) - AG(\overline{z})\|^2  + 2\langle G(\widehat{z}) - G(\overline{z}), A^T(AG(\overline{z})- y) \rangle,
\end{aligned}
\end{equation*}
which hence
\begin{equation}\label{eql}
\begin{array}{ll}
\frac{1}{m}\|AG(\widehat{z}) - AG(\overline{z})\|^2
&\leq 2\langle G(\widehat{z}) - G(\overline{z}), \frac{1}{m}A^T(y -AG(\overline{z})) \rangle\\ [1.5ex]
&\leq 2\|G(\widehat{z}) - G(\overline{z})\|_1\|\frac{1}{m}A^T(y -AG(\overline{z}))\|_\infty \\[1.5ex]
&\leq 4 \|\frac{1}{m}A^T(y -AG(\overline{z}))\|_\infty,
\end{array}
\end{equation}
where the last step is from the assumption $G(\mathcal{B}_2^k(1))\subset \mathcal{B}_1^{n} \Rightarrow \|G(\widehat{z}) - G(\overline{z})\|_1\leq 2$. Next we bound $\frac{1}{m}\|A^T(y - AG(\overline{z}))\|_\infty$.
By triangle inequality,
\begin{equation}\label{the5}
\begin{array}{ll}
\frac{1}{m}\|A^T(y - AG(\overline{z}))\|_\infty
&= \frac{1}{m}\|A^T(y - A\widetilde{x^*} + A\widetilde{x^*} - AG(\overline{z}))\|_\infty \\[1.5ex]
&\leq\frac{1}{m}\|A^T(y - A\widetilde{x^*})\|_\infty + \frac{1}{m}\|A^T(A\widetilde{x^*} - AG(\overline{z}))\|_\infty.
\end{array}
\end{equation}
The first term in \eqref{the5} can be estimated by
\begin{equation}\label{theq7}
\begin{array}{ll}
\frac{1}{m}\|A^T(y - A\widetilde{x^*})\|_\infty
&=\|\frac{1}{m}A^Ty -\Sigma\widetilde{x^*} + \Sigma\widetilde{x^*} - \frac{1}{m}A^TA\widetilde{x^*}\|_\infty\\[1.5ex]
&\leq\|\frac{1}{m}A^Ty -\Sigma\widetilde{x^*}\|_\infty + \|\Sigma\widetilde{x^*} - \frac{1}{m}A^TA\widetilde{x^*}\|_\infty\\[1.5ex]
&\leq\frac{1}{m}\|A^Ty - \mathbb{E}[A^Ty]\|_\infty + \|\Sigma - \frac{1}{m}A^TA\|_\infty\|\widetilde{x^*}\|_1\\[1.5ex]
&= \frac{1}{m}\|\sum_{i=1}^m (A_iy_i-\mathbb{E}[A_iy_i])\|_\infty + \|\Sigma - \frac{1}{m}A^TA\|_\infty\|\widetilde{x^*}\|_1\\[1.5ex]
& \leq \mathcal{O}\left(\sqrt{\frac{\log n}{m}}\right),
\end{array}
\end{equation}
where the last inequality is from Lemma \ref{linf}. To estimate the second term in \eqref{the5}, denote by $\widetilde{\Delta} = \widetilde{x^*} - G(\overline{z})$,
we then have
\begin{equation}\label{the8}
\begin{array}{ll}
\frac{1}{m}\|A^T(A\widetilde{x^*} - AG(\overline{z}))\|_\infty
&= \frac{1}{m}\|A^TA\widetilde{\Delta}\|_\infty\\[1.5ex]
&= \|\frac{1}{m}A^TA\widetilde{\Delta} - \Sigma\widetilde{\Delta} + \Sigma\widetilde{\Delta}\|_\infty\\[1.5ex]
&\leq \|(\frac{1}{m}A^TA - \Sigma)\widetilde{\Delta}\|_\infty + \|\Sigma\widetilde{\Delta}\|_\infty\\[1.5ex]
&\leq \|\widetilde{\Delta}\|_1(\|\frac{1}{m}A^TA - \Sigma\|_\infty + \|\Sigma\|_\infty)\\[1.5ex]
&\leq \mathcal{O}\left(\sqrt{\frac{\log n}{m}}+1\right)\tau.
\end{array}
\end{equation}

From lemma \ref{A_REC}, $A$ satisfies the S-REC$(G(\mathcal{B}_2^k(1)), 0.5\sqrt{\sigma_{min}(\Sigma)}, O(\delta)),$ with probability $1 - e^{-O(m/k)}$ as long as
 $m = O(k\log\frac{L}{\delta})$, i.e.,
 \begin{equation}\label{the3}
   \frac{1}{m}\|AG(\widehat{z}) - AG(\overline{z})\|^2
   \geq 0.5\sqrt{\sigma_{min}(\Sigma)}\|G(\widehat{z}) - G(\overline{z})\|^2 - \mathcal{O}(\delta).
   \end{equation}
Substituting \eqref{the5} - \eqref{the3} into \eqref{eql} we obtain

\begin{equation}\label{eq2}
0.5\sqrt{\sigma_{min}(\Sigma)}\|G(\widehat{z}) - G(\overline{z})\|^2 - \mathcal{O}(\delta) \leq \mathcal{O}\left(\sqrt{\frac{\log n}{m}}+\tau + \sqrt{\frac{\log n}{m}}\tau\right).
\end{equation}
We may choose $\delta = O(\tau)$ in \eqref{eq2} and substituted it into \eqref{the1}, we conclude
\begin{equation*}
\|G(\widehat{z}) - \widetilde{x}^*\| \leq \mathcal{O}\left(\left(\frac{\log n}{m}\right)^{1/4} + \sqrt{\tau}\right).
\end{equation*}
\end{proof}

Obviously, $\tau$ measures the approximation error between the target $x^*$ and the generator $G$. If we assume that $\tau$ is smaller than $\mathcal{O}((\frac{\log n}{m})^{1/2})$,
Theorem \ref{thel1} shows that under that approximate low generative dimension prior, our proposed least decoder  (\ref{ls1})-(\ref{ls2})
can achieve an estimation error  $\mathcal{O}((\log n/m)^{1/4})$  provide that the number of samples $m \geq \mathcal{O}(\max \{\log n, k\log \frac{L}{\tau}\})$.
 Similar results has been established for 1-bit CS under the sparsity prior $\|x^*\|_0\leq s$ in the literatures. For example,
\cite{PlanVershynincpam:2013} proposed  a linear programming  decoder
 \begin{equation*}
 x_{\mathrm{lp}}\in \arg\min_{x\in \mathbb{R}^n} \|x\|_1 \quad \mathrm{s.t.} \quad  y \odot A x\geq 0  \quad \|A x\|_1  =  m.
 \end{equation*}
 in the   noiseless setting  without sign flips. It has been proved in \cite{PlanVershynincpam:2013} that
 $$\|\frac{x_{\mathrm{lp}}}{\|x_{\mathrm{lp}}\|} - x^*\|\leq \mathcal{{O}}((\frac{s\log n}{m})^{1/5}),$$
 provided that $m=\mathcal{O}(s\log^2(n/s))$.
Later,  in  \cite{PlanVershynin:2013}, another convex decoder
  \begin{equation*}
 x_{\mathrm{cv}} \in \arg\min_{x\in \mathbb{R}^n} -\langle y, A x\rangle/m \quad \mathrm{s.t.} \quad \|x\|_1 \leq s,  \quad \|x\| \leq 1,
 \end{equation*}
 is shown to achieve a estimation error bound
$$\|\frac{x_{\mathrm{cv}}}{\|x_{\mathrm{cv}}\|} - x^*\|\leq \mathcal{{O}}((\frac{s\log n}{m})^{1/4}).$$

Although the order of estimation error   proved in Theorem  \ref{thel1} does  not depend on the Lipschtiz constant of the generator $G$  which is usually  exponential order of the depth of   the  neural networks  \cite{bora2017compressed},  it  is  sub-optimal.  Next we  improve the estimation error bound  by using the tool of local (Gaussian) mean width. The definition of local mean width is given below, it can also be found in \cite{PlanVershynin:2016,PlanVershynin:2017}.
\begin{definition}\label{gmw}
Let $S\subseteq\mathbb{R}^n$. The local mean width of $S$ is a function of scale $t \geq 0$ defined as
$$
\omega_{t}(S)=\mathbb{E}_{g\sim \mathcal{N}(\mathrm{0},\mathbf{I})} \left[\sup _{x \in S \cap t \mathcal{B}_2^{n}(1)}\langle x, g\rangle\right].
$$
\end{definition}

\begin{theorem}\label{theorem2}
Given an $L$-Lipschitz generator satisfying  $G(\mathcal{B}_2^k(r))\subset \mathcal{S}_2^{n-1}$.
Assume the 1-bit CS model \eqref{setup} holds with  $x^* \in \mathcal{G}_{k,\tau}(r)$,
and  $m \geq \mathcal{O}\left(\max\{k\log\frac{Lrn}{k},\log n\}\right)$,
 then with high  probability, 
 the least square decoder defined in (\ref{ls2})-(\ref{ls1}) satisfies
\begin{equation*}
\|\hat{x}-c{x^*}\|_2\leq \mathcal{O}\left(\sqrt{\frac{k}{m}\log\frac{rLn}{k\gamma}}\right) +\mathcal{O}(\frac{\tau n}{m}),
\end{equation*}
for $\gamma=\max\{\tau,\frac{k}{m}\log \frac{Lrn}{k}+\sqrt{\frac{\log n}{m}}\}$. If the approximation error satisfies $\tau = \mathcal{O} (\frac{\sqrt{mk\log (Ln)}}{n})$ and $r = \mathcal{O}(1)$, then we have
\begin{equation*}
\|\hat{x}-c{x^*}\|_2\leq \mathcal{O}\left(\sqrt{\frac{k}{m}\log(Ln)}\right).
\end{equation*}
\end{theorem}

First we do some preparing work before the proof. Similar as the proof to Theorem \ref{thel1}, let $\overline{z}\in \mathcal{B}_{2}^{k}(r)$ satisfying  $\|G(\overline{z})-\widetilde{x^*}\| \leq \tau$. By triangle inequality \eqref{the1}, we have
\begin{equation}\label{eqth21}
\|G(\widehat{z})-\widetilde{x^*}\| \leq \|G(\widehat{z})- G(\overline{z})\|+ \tau.
\end{equation}
Let $h = G(\hat{z})-G(\overline{z})$, $\gamma=\max\{\tau,\frac{k}{m}\log \frac{Lrn}{k}+\sqrt{\frac{\log n}{m}}\}$, $S = G(\mathcal{B}_2^k(r))$, and $D_{\gamma}(S,G(\overline{z}))$ be the  tangent cone which is defined by
$$ D_{\gamma}(S,G(\overline{z})) = \{t u:  t>0, u = G(z)-G(\overline{z}), \|u\|>\gamma\}.$$
If $\|h\|\leq \gamma $ this Theorem is trivial by (\ref{eqth21}), otherwise $\|h\|> \gamma $, then $h\in D_{\gamma}(S,G(\overline{z}))$.
Let $$\mathcal{D} = D_{\gamma}(S,G(\overline{z})) \cap \mathcal{S}_{2}^{n-1}.$$
We need the following two lemmas to proceed the proof.
\begin{lemma}\label{pv}
 With probability at least $0.99$, both
 \begin{equation}\label{eqth25}
 \inf _{v \in \mathcal{D}} \frac{1}{\sqrt{m}}\|A v\|_{2} \geq C_0
 \end{equation}
 \begin{equation}\label{eqth26}
 \sup _{v \in \mathcal{D}}\frac{1}{m}\left\langle v, A^{T} (y -A\widetilde{x^*})\right\rangle \leq C\frac{\omega_1(\mathcal{D})}{\sqrt{m}}.
 \end{equation}
 hold, where $\omega_1(\mathcal{D})$ is the local (Gaussian) mean width of $\mathcal{D}$ given  in Definition \ref{gmw}.
 \end{lemma}
\begin{proof}
The results can be found in the proof to Theorem 1.4 in \cite{PlanVershynin:2016}.
\end{proof}

 \begin{lemma}\label{gw}
$$\omega_1(\mathcal{D})=  \mathcal{O}\left(\sqrt{k\log(\frac{rLn}{k\gamma})}\right).$$
 \end{lemma}
\begin{proof}
Recall that $$D_{\gamma}(S,G(\overline{z})) = \{t u:  t>0, u = G(z)-G(\overline{z}), \|u\|>\gamma\}.$$
  and  $$\mathcal{D} = D_{\gamma}(S,\overline{z}) \cap \mathcal{S}_{2}^{n-1}.$$
  Then $\mathcal{D}=\{\frac{ G(z)-G(\overline{z})}{\| G(z)-G(\overline{z})\|}: z\in \mathcal{B}_2^k(r), \| G(z)-G(\overline{z})\|>\gamma\}.$
Let $\mathcal{U}$ be $\frac{\epsilon\gamma}{2L}-$ net on $\mathcal{B}_2^k(r)$ satisfying $$\log|\mathcal{U}|\leq k\log(\frac{8Lr}{\gamma\epsilon}),$$ which can be obtained by Lemma \ref{cov}.
Then, $\mathcal{C} = \{\frac{G(u)-G(\overline{z})}{G(u)-G(\overline{z})}:u \in \mathcal{U} \ \ \mathrm{and} \ \ \|G(u)-G(\overline{z})\|\geq \gamma\}$ is a $\epsilon-$ net of $\mathcal{D}$. 
Indeed,
let $\frac{a}{\|a\|}$ with $a =  G(z)-G(\overline{z})$ be an arbitrary element in $\mathcal{D}$, and $u\in \mathcal{U}$ such that
$\|u-z\|\leq \frac{\epsilon\gamma}{2L}$. Let   $b=  G(u)-G(\overline{z})$,  then $\frac{b}{\|b\|}\in \mathcal{C}$ and satisfies that

\begin{align*}
& \left\| \frac{a}{\|a\|}-\frac{b}{\|b\|}\right\| = \left\|\frac{a\|b\|-b\|a\|}{\|a\|\|b\|}\right\|  \\
& \leq \frac{\|(a-b)\|b\|\|+\|b(\|b\|-\|a\|)\|}{\|a\|\|b\|}\\
& \leq 2\frac{\|a-b\|}{\|a\|} \leq 2L\frac{\frac{\epsilon\gamma}{2L}}{\gamma}=\epsilon,
\end{align*}
where in last inequality we use the facts that $G$ is $L$-Lipschitz and $\|a\|\geq \gamma.$
By Massart's finite class Lemma in \cite{boucheron2013concentration}, the  local Gaussian  width of $\mathcal{C}$  satisfies
\begin{equation}\label{gf}
\omega_1(\mathcal{C})\leq \sqrt{2k\log(\frac{16Lr}{\gamma\epsilon})}.
\end{equation}
Since $\forall x \in \mathcal{D}$, there exist $\tilde{x} \in \mathcal{C}$ such that
$\|x-\tilde{x}\|\leq \epsilon$. We then have $\forall g \sim \mathcal{N}(\mathbf{0}, \mathbf{I})$
\begin{equation*}
\langle g, x\rangle  \leq \langle g, \tilde{x}\rangle + \langle g, x-\tilde{x}\rangle \leq  \langle g, \tilde{x}\rangle + \epsilon \|g\|.
\end{equation*}
The above display and the definition of local Gaussian mean width and the fact $\mathcal{D}\subset \mathcal{B}_2^{n}(1)$
implies
\begin{align*}
\omega_1(\mathcal{D})&= \mathbb{E}[\sup_{x\in \mathcal{D}}\langle g, x\rangle]\leq \mathbb{E} [\sup_{\tilde{x} \in \mathcal{C}}\langle g, \tilde{x}\rangle + \epsilon \|g\|]\\
& \leq \omega_1(\mathcal{C})+\sqrt{n}\epsilon \\
&\leq  \sqrt{2k\log(\frac{16Lr}{\gamma\epsilon})} + \sqrt{n}\epsilon,
\end{align*}
where the second equality follow from $\mathbb{E}[\|g\|] = \sqrt{n}$, and in the third inequality  we use \eqref{gf}.
The proof will be finished by  setting $\epsilon = \sqrt{\frac{k}{n}}$.

\end{proof}

 \begin{lemma}\label{spectralnorm}
Let $a_i\in \mathbb{R}^{n}, i = 1,...n$ are i.i.d samples with mean $0$ and covariance matrix $\Sigma$. Denote $\Sigma_m = \sum_{i=1}^{m} a_ia_i^{T}/m$. Then for any $u \geq 0$
$$
\left\|\Sigma_{m}-\Sigma\right\| \leq \mathcal{O}\left(\sqrt{\frac{n+u}{m}}+\frac{n+u}{m}\right)\|\Sigma\|
$$
with probability at least $1-2 e^{-u}$.
 \end{lemma}
 \begin{proof}
 See  exercise 4.7.3 in \cite{vershynin2018high}.
 \end{proof}

Now we can move to the proof to Theorem \ref{theorem2}.

\begin{proof}
Similar as \eqref{eql} in the proof to Theorem \ref{thel1}, by \eqref{eqth25} in Lemma \ref{pv} and triangle inequality, we have
with probability  at least $0.99$ that
\begin{equation}\label{eq23}
\begin{array}{ll}
C_0 \|h\|^2 \leq \frac{1}{m}\|Ah\|^2
&\leq 2\langle h, \frac{1}{m}A^T(y -AG(\overline{z}))\rangle.\\[1.5ex]
&\leq 2|\langle h, \frac{1}{m}A^T(y -A\widetilde{x^*}) \rangle|
+ 2|\langle h, \frac{1}{m}A^T A(\widetilde{x^*}-G(\overline{z})) \rangle|.
\end{array}
\end{equation}
 We have to bound the two terms in \eqref{eq23}. For the first term, let $v = \frac{h}{\|h\|} = \frac{G(\widehat{z}) - G(\overline{z})}{\|G(\widehat{z}) - G(\overline{z})\|}$, hence $v\in \mathcal{D}$. Then by \eqref{eqth26} in Lemma \ref{pv} and Lemma \ref{gw}, we obtain that with probability  at least $0.99$
\begin{equation}\label{eq30}
|\langle h, \frac{1}{m}A^T(y -A\widetilde{x^*}) \rangle|  \leq \mathcal{O}\left(\sqrt{\frac{k\log(rLn/(k\gamma))}{m}}\right)  \|h\|.
\end{equation}
For the second term in \eqref{eq23}, we apply Cauchy-Schwarz inequality and using spectral norm estimation for random matrix in Lemma \ref{spectralnorm}  , we get with high probability at least $1-e^{-n}$ that
\begin{align}\label{eq31}
 |\langle h, \frac{1}{m}A^T A(\widetilde{x^*}-G(\overline{z})) \rangle|
 &\leq \|h\|  \frac{1}{m}\|A^T(A\widetilde{x^*} - AG(\overline{z}))\| \nonumber\\
 &\leq (\|A^TA/m-\Sigma\|_2+\|\Sigma\|) \|\widetilde{x^*}-G(\overline{z})\| \|h\| \nonumber \\
 &\leq \mathcal{O}(\sqrt{\frac{2n}{m}}+\frac{2n}{m}+1)\|\Sigma\|\tau\|h\|\\
 & \leq \mathcal{O}(\tau \frac{n}{m})\|h\|.
 \end{align}

Combining \eqref{eq23}, \eqref{eq30} and \eqref{eq31} we get
$$\|h\|\leq \mathcal{O}(\sqrt{\frac{k\log(rLn/(k\gamma))}{m}}) +\mathcal{O}(\frac{\tau n}{m}).$$
Moreover, if the approximation error $\tau = \mathcal{O} (\frac{\sqrt{mk\log (Ln)}}{n})$ and $r = \mathcal{O}(1)$, the above inequality can be reduced to
\begin{equation*}
\|\hat{x}-c{x^*}\|_2\leq \mathcal{O}\left(\sqrt{\frac{k}{m}\log(Ln)}\right),
\end{equation*}
which completes the proof.
\end{proof}

By assuming the   Lipschitz constant $L$ is  larger than $n$ (this usually  holds in deep neural network generators), the estimation
error $\mathcal{O}(\sqrt{\frac{k\log L}{m}})$ and the sample complexity $m \geq \mathcal{O}(k\log L) $ proved in Theorem \ref{theorem2} are  sharp
even in  the standard compressed sensing with generative prior \cite{liu2020information}.
 Under generative prior,   \cite{liu2020sample} proposed
the estimator   $\hat{x} = G(\hat{z})$ with $\hat{z} \in \{z: y=\mathrm{sign}(AG(z))\}$ in the setting the rows of $A$   are   i.i.d.   sampled from   $\mathcal{N}(\textbf{0},\mathbf{I})$.   The sample complexity obtained  in \cite{liu2020sample}
 is also $\mathcal{O}(k\log L)$.
 \cite{qiu2020robust} proposed unconstrained empirical
risk minimization to recover the  1-bit CS in the scenario that   the rows of $A$   are   i.i.d.   sampled from   subexponential distributions, and the generator $G$ is restricted to be  a  $d$-layer ReLU network. The sample complexity derived in \cite{qiu2020robust} is $\mathcal{O}(kd\log n)$.

There are some works on generative priors which assumes that the target signals can be exactly generated by a generator $G$, i.e.,  $x^* \in \mathcal{G}_{k,\tau}(r)$ with $\tau=0$, see e.g. 
\cite{liu2020generalized,qiu2020robust}.
As mentioned in Theorem \ref{theorem2}, we can relax this assumption by requiring
the  target signal $x^*$ can  be generated  by $G$  approximately, i.e.,
$x^* \in \mathcal{G}_{k,\tau}(r)$  with
\begin{equation}\label{eq:tau}
\tau = \mathcal{O} (\frac{\sqrt{mk\log (Ln)}}{n}).
\end{equation}

Since natural signals/images data  with low intrinsic dimension can be represented   approximately by   neural networks is empirically verified in  \cite{goodfellow14,kingma14,rezende2015variational}.
 Next, we verify the assumption \eqref{eq:tau} by construct   a generator  ${G}$ with  properly chosen depth and width  based on  the recent approximation ideas of  deep neural networks \cite{shen2019nonlinear,vershynin2020memory,vardi2021optimal,huang2021error} by  utilizing the bit extraction techniques \cite{bartlett1999almost,bartlett2019nearly}. To this end, we recall the
definition of  Minkowski dimension  which is used to measures the intrinsic dimension of the target signals living in a large ambient dimension.

\begin{definition}\label{Minkowski dimensions}
\textnormal{
The upper and the lower Minkowski dimensions of a set $A \subseteq \bR^n$ are defined respectively as
\begin{align*}
\overline{\dim}_M(A) := \limsup_{\epsilon\to 0} \frac{\log |N_{\epsilon}|}{-\log \epsilon}, \\
\underline{\dim}_M(A) := \liminf_{\epsilon\to 0} \frac{\log |N_{\epsilon}|}{-\log \epsilon},
\end{align*}
where $N_{\epsilon}$ is the $\epsilon$-net of $A$.
If $\overline{\dim}_M(A) = \underline{\dim}_M(A) = \dim_M(A)$, then $\dim_M(A)$ is called the \emph{Minkowski dimension} of the set $A$.
}
\end{definition}

The Minkowski dimension measures how the  number of elements in the $\epsilon$-net $N_{\epsilon}$ of  $A$ decays when the radius of covering balls converges to zero. We collect the useful properties from  \cite{falconer2004fractal} of Minkowski dimension.
\begin{proposition}\cite{falconer2004fractal}\label{pmd}
$\dim_M(A) < \bar{n}$ if and only if $\forall \gamma >0$, $$\epsilon^{-(\bar{n}-\gamma)}\leq |N_{\epsilon}| \leq \epsilon^{-(\bar{n}+\gamma)}$$ holds when $\epsilon$ small enough.
Furthermore,    $\dim_M(A) \leq \bar{n}$ implies that
$\forall \epsilon>0$, there exist an $\epsilon$-net $N_{\epsilon}$ of $A$  such that $|N_{\epsilon}| \leq c\epsilon^{-\bar{n}}$, where $c$ is a  finite number.
\end{proposition}

The next three Lemmas present the approximation ability of the deep neural networks.

 \begin{lemma}\label{fit1}
 For any $\mathcal{W},\ell\in \mathbb{N}$, given $\mathcal{W}^{2} \ell$ samples $\left(z_{i}, y_{i}\right), i=1, \ldots, \mathcal{W}^{2} \ell$, with distinct $z_{i} \in \mathbb{R}^{k}$ and $y_{i} = \sum_{j=1}^{\ell}2^{-j}b_{i,j}$, $b_{i,j}\in\{0,1\}$. There exists a ReLU network $G_1$ with width $4 \mathcal{W}+4$ and depth $\mathcal{\ell}+2$ such that $G_1\left(z_{i}\right)=y_{i}$ for $i=1, \ldots, \mathcal{W}^{2} \ell.$
 \begin{proof}
 This lemma   follows directly from  of Lemma $2.1$ and $2.2$ in
\cite{shen2019nonlinear}.
 \end{proof}
 \end{lemma}
 \begin{lemma}\label{bitext}
 For any $\ell \in \mathbb{N}$, there exists a ReLU network $G_2$ with width $8$ and depth $2\ell$  such that $G_2(x, j)=b_{j}$ for any $x=\sum_{j=1}^{\ell}2^{-j}b_j$  with $b_{j} \in\{0,1\}$ and $j=1,2, \ldots, \ell$.
 \end{lemma}
 \begin{proof}
 This lemma follow from  Lemma 5.7 in \cite{huang2021error}.
 \end{proof}
 \begin{lemma}\label{intmul}
  Let $\mathcal{W}\in \mathbb{N}$. Given any $\mathcal{W}^{2} \ell^2$ points  $\{\left(z_{i},  b_{i,j}\right) i=1, \ldots, \mathcal{W}^{2} \ell,  j=1,...,\ell\}$, where $z_{i} \in \mathbb{R}^{k}$ are distinct and $b_{i,j} \in \{0,1\}$. There exists a ReLU network $G_3$ with width $4 \mathcal{W}+6$ and depth $3 \ell+1$ such that $G_3\left(z_{i},j\right)=b_{i,j}$, $i=1, \ldots, \mathcal{W}^{2} \ell,   j=1,...,\ell$.
  \end{lemma}
  \begin{proof}
    $\forall i=1, \ldots, \mathcal{W}^{2} \ell$, let $y_{i}= \sum_{j=1}^{\ell}2^{-j}b_{i,j} \in[0,1].$  By  Lemma \ref{fit1} there exists a network $G_{1}$ with width $4 \mathcal{W}+4$ and depth $\ell+2$ such that $G_{1}\left(z_{i}\right)=y_{i}$ for $i=1, \ldots, \mathcal{W}^{2} \ell$.
By Lemma 6.4, there exists a network $G_{2}$ with width 8 and depth $2\ell$ such that $G_{2}\left(y_{i}, \ell\right)=b_{i, \ell}$ for any $i=1, \ldots, \mathcal{W}^{2} \ell$ and $j=1, \ldots, \ell$. Therefore, the function $G_3(\cdot, j)=G_{2}\left(G_{1}(\cdot), j\right)$  implemented by a ReLU network with width $ 4\mathcal{W}+6$ and depth $3 \ell+1$ satisfies our requirement.
  \end{proof}

\begin{theorem}\label{thapp}
Assume the target signals $x^*\in A^*\subseteq [0,1]^n$ with $\dim_M(A^*) = k$. Then $\forall \tau \in (0,1)$ there exist a generator network $G:\mathbb{R}^{k}\rightarrow \mathbb{R}^n$
with depth $3\ell+2$  and width  $(4\lceil \sqrt{sn/\ell}\rceil+6)n$ such that $x^*\in \mathcal{G}_{k,\tau,2}(1), \forall x^*\in A^*$, where
$\ell = \lceil\log_2(\frac{2n}{\tau})\rceil+1, \quad s = \mathcal{O}(\tau^{-k}).$
\end{theorem}
\begin{proof}
Let $\epsilon = \tau/2$. Since the target signals $x^*$ are contained in $A^*$ with   $\dim_M(A^*) = k$, then there exist an $\epsilon$-net $N_{\epsilon} = \{o^{*}_i\}_{i=1}^{s}$ of $A^*$ with $s\leq c\epsilon^{-k}$ by Proposition \ref{pmd}.
For any $o^{*}_i \in N_{\epsilon}$, let the binary representation of $ o^{*}_i$ be $o^{*}_i=\sum_{j=1}^{\infty} 2^{-j} \tilde{o}^{*}_{i,j}$ with $\tilde{o}^{*}_{i,j}\in\mathbb{R}^n$ whose entries $\in \{0,1\}$.
 Let $\ell = \lceil\log_2(\frac{n}{\epsilon})\rceil+1$, the truncation of $o^{*}_i$ be $T_{\ell}o^{*}_i = \sum_{j=1}^{\ell} 2^{-j} \tilde{o}^{*}_{i,j}$,
then it implies that \begin{equation}\label{tc}
\|o^{*}_i-T_{\ell}o^{*}_i\|\leq \epsilon, \forall i =1,...,s.
\end{equation}
By construction of $N_{\epsilon}$,  (\ref{tc}) and triangle inequality, we have $\{T_{\ell}o^{*}_i\}_{i=1}^s$ is an $\tau$-net of $A^*$.
  Let $e= (1,0,0,...0)^T\in\mathbb{R}^s$, $\mathcal{W} = \lceil \sqrt{sn/\ell}\rceil$,  and  $z_i$ be the $i$-th element of $\{e,  e/2,...,e/(ns)\}$,  $b_{i,j} = G_2(T_{\ell}o^{*}_i,j)$,
   $i=1,...,sn, j=1,...\ell$.
  By Lemma \ref{intmul}, we have $G_3\left(z_{i},j\right)=b_{i,j}$, $i=1, \ldots, sn,   j=1,...,\ell$.
  $\forall x\in \mathbb{R}^k$, define  $$G(x)=(\sum_{j=1}^{\ell} 2^{-j}G_3(a_1 x, j),\sum_{j=1}^{\ell} 2^{-j}G_3(a_2x, j),...\sum_{j=1}^{\ell} 2^{-j}G_3(a_n x,j))^T:\mathbb{R}^k\rightarrow \mathbb{R}^n$$ with $a_1>0, a_2>0,...,a_n>0.$
  Let $\theta_1$ and $\theta_2$ be the parameters of the ReLU network $G_1$ and $G_2$, respectively.
   Denote $a =(a_1,...,a_n)\in\mathbb{R}^n$ and $\theta=(a,\theta_1,\theta_2)$.
  Then $G(x)$ is a ReLU network with free parameter $\theta$ and depth $3\ell+2$  and width  $(4\lceil \sqrt{sn/\ell}\rceil+6)n$. We use $G_{\theta}(x)$ to emphasize the dependence of $G$ on $\theta$.
 For $i=1,2,...s$, let  $a^{(i)} = (\frac{1}{(i-1)n+1},\frac{1}{(i-1)n+2}, ..., \frac{1}{(i-1)n+n})^T \in\mathbb{R}^n,$ 
 $\theta^{(i)}  = (a^{(i)},\theta_1,\theta_2)$,   by construction, we have
    $G_{\theta^{(i)}} (e) = T_{\ell}o^{*}_{i}$.
\end{proof}

\section{Numerical Experiments}

\subsection{Experiments setting}
The rows of the matrix $A$ are i.i.d. random vectors sampled from the multivariate normal distribution $\mathcal{N}(\mathbf{0},\Sigma)$ with $\Sigma_{jk}=\nu^{|j-k|}$, $1\leq j,k\leq n$, $\nu =0.3$ in our tests. The elements of $\epsilon$ are generated from $\mathcal{N}(\mathbf{0},\sigma^2\mathbf{I}_m)$ with $\sigma=0.1$ in our examples. $\eta$ has independent coordinates with $\mathbb{P}\{\eta_i=1\}=1-\mathbb{P}\{\eta_i=-1\}=q\neq\frac{1}{2}$, with different $q$ which will be clarified in each example.
The generative model $G$ in our experiments is a pretrained variational autoencoder (VAE) model\footnote{We use the pre-trained generative model of (Bora et., 2017) available at https://github.com/
AshishBora/csgm.}. The MNIST dataset \cite{lecun1998gradient} consisting $60000$ handwritten images of size 28$\times$28 is applied in our tests. For this dataset, we set the VAE model with a latent dimension $k=20$. Input to the VAE is a vectorized binary image of input dimension $784$. Encoder and decoder are both fully connected network with two hidden layers, i.e., encoder and decoder are with size $784-500-500-20$ and $20-500-500-784$, respectively.

To avoid the  norm constraint $\|z\|_2 \leq r$ in  the least square decoder (\ref{ls1})-(\ref{ls2}), we use its Lagrangian form as following:
\begin{equation}\label{LS generative model}
\min_{z} \frac{1}{2m}\|y - AG(z)\|^2 + \lambda\|z\|^2,
\end{equation}
where the regularization parameter $\lambda$ is chosen as $0.001$ for all the experiments.
We do $10$ random restarts with $1000$ steps per restart and choose the best estimation. The reconstruction error is calculated over $10$ images by averaging the per-pixel error in terms of the $l_2$ norm.
\subsection{Experiment Results}
We compare our results with two SOTA algorithms: BIHT \cite{JacquesLaska:2013}
and generative prior based algorithm  VAE \cite{liu2020sample}.
The least square decoder with VAE in our paper is named by LS-VAE.

Figures $1-4$ indicate that with or without sign flip in measurements,
generative prior based methods attain more accurate reconstruction than BIHT.
Additionally, if sign flips are added, Figures $3$ and $4$ show that LS-VAE
attain the higher accurate reconstruction.

In Figure $5$, we plot the reconstruction error for different measurements (from $50$ measurements to $300$ measurements).
VAE and LS-VAE both have smaller reconstruction errors, but LS-VAE is slightly better.
Moreover, after $200$ measurements, the reconstruction error emerges saturation for generative prior based methods, due to
its output is constrained to the presentation error \cite{bora2017compressed}.
\begin{figure}[ht!]
\begin{center}
\includegraphics[scale=0.4]{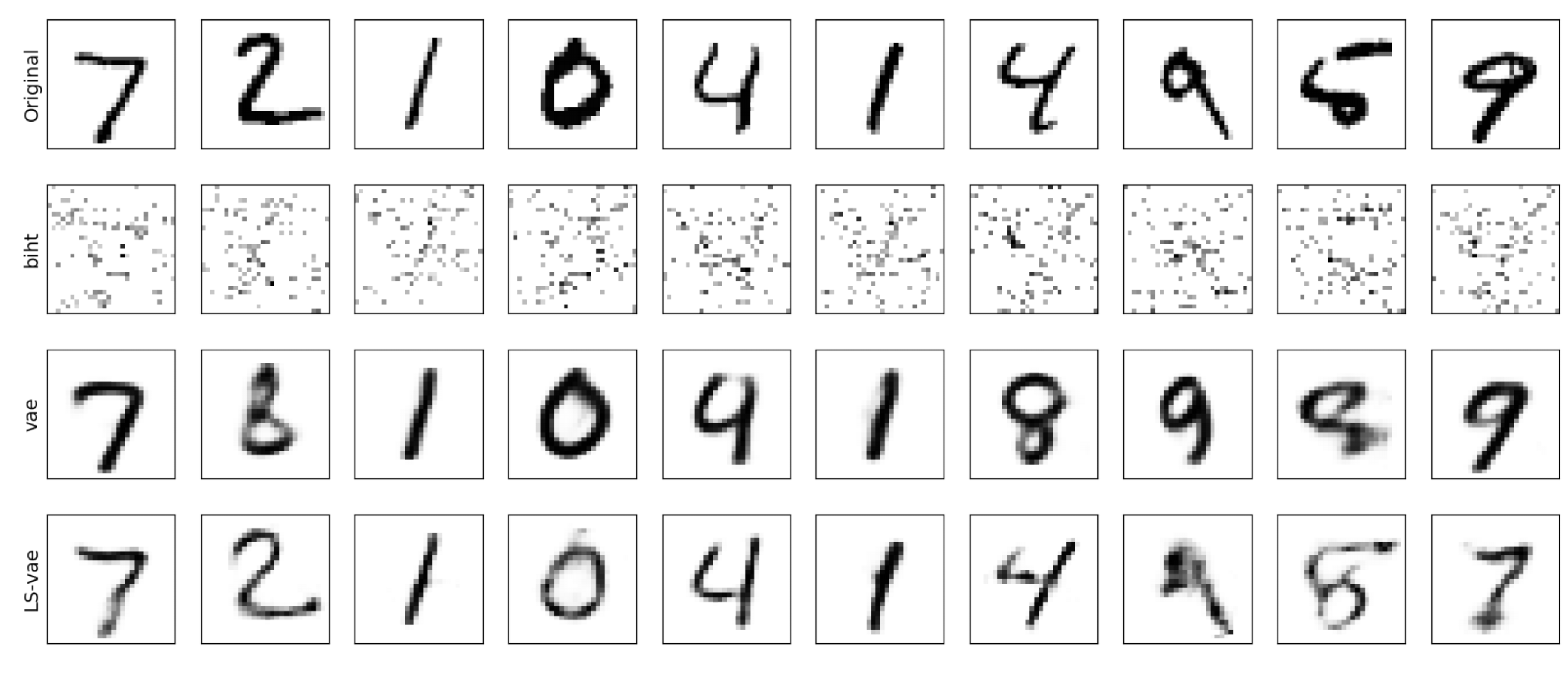}
\setlength{\abovecaptionskip}{1pt}
\caption{original images, reconstructions by BIHT, VAE and LS-VAE (from top to bottom row)
with 100 measurements}
\end{center}
\end{figure}
\begin{figure}[ht!]
\begin{center}
\includegraphics[scale=0.4]{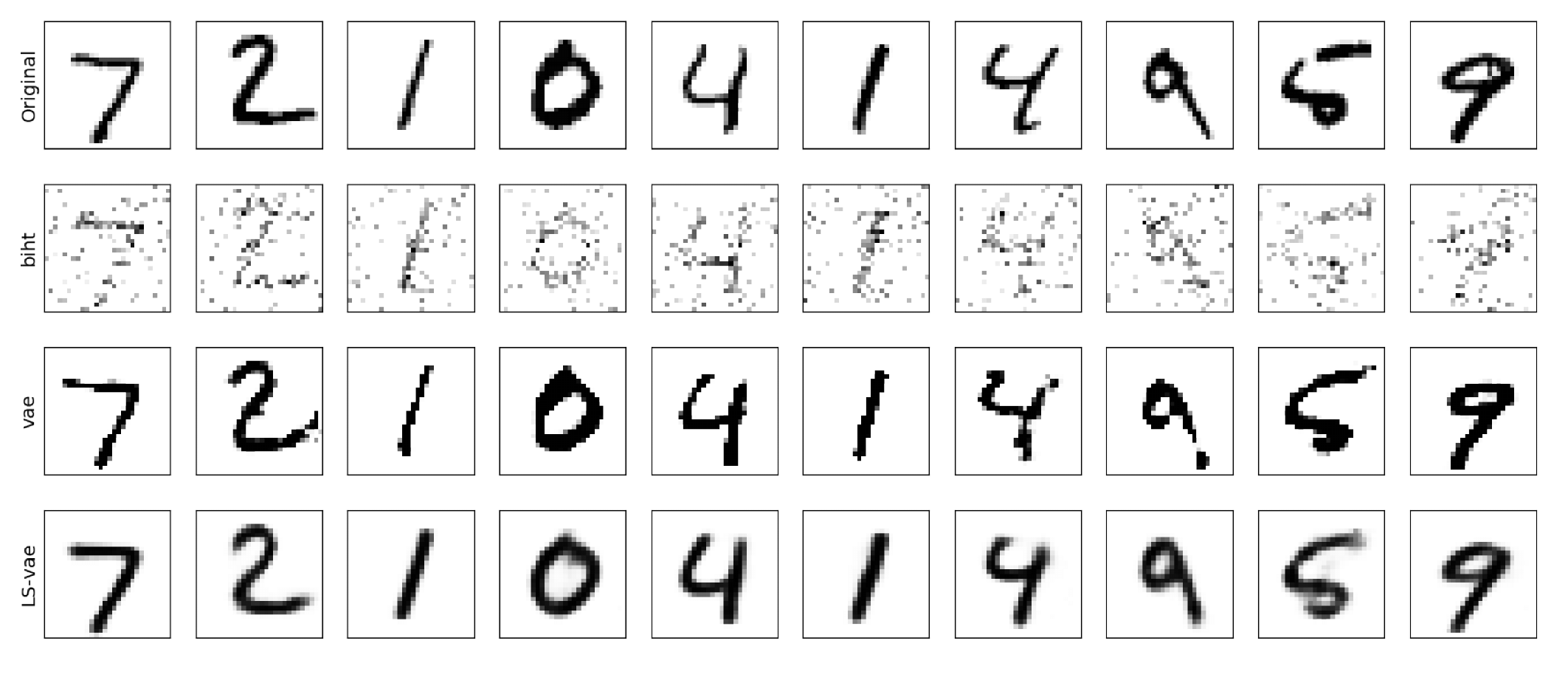}
\setlength{\abovecaptionskip}{1pt}
\caption{original images, reconstructions by BIHT, VAE and LS-VAE (from top to bottom row) with 300 measurements}
\end{center}
\end{figure}
\begin{figure}[ht!]
\begin{center}
\includegraphics[scale=0.42]{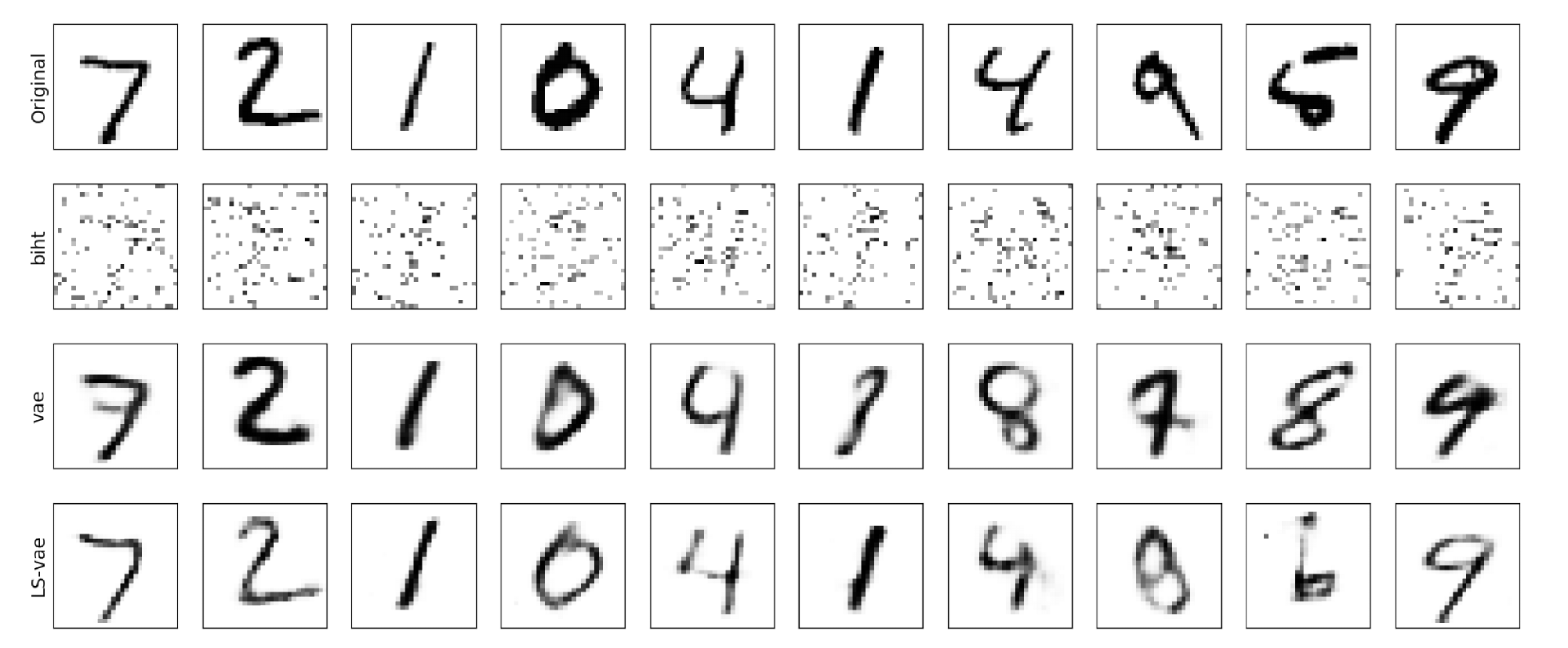}
\setlength{\abovecaptionskip}{1pt}
\caption{original images, reconstructions by BIHT, VAE and LS-VAE (from top to bottom row)  with 100 measurements
and 3\% sign flips}
\end{center}
\end{figure}
\begin{figure}[ht!]
\begin{center}
\includegraphics[scale=0.4]{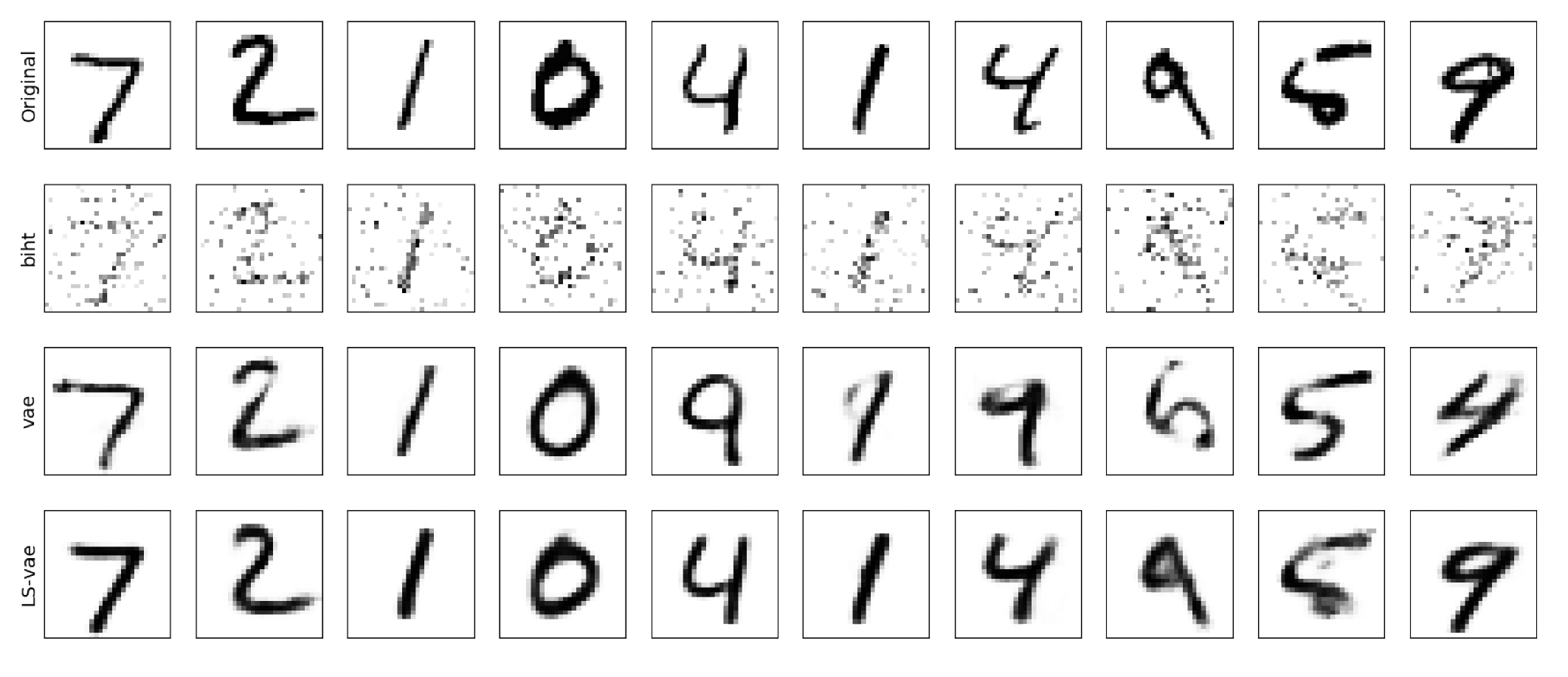}
\setlength{\abovecaptionskip}{1pt}
\caption{original images, reconstructions by BIHT, VAE and LS-VAE (from top to bottom row)  with 300 measurements
and 3\% sign flips}
\end{center}
\end{figure}
\begin{figure}[ht!]
\begin{center}
\includegraphics[scale=0.5]{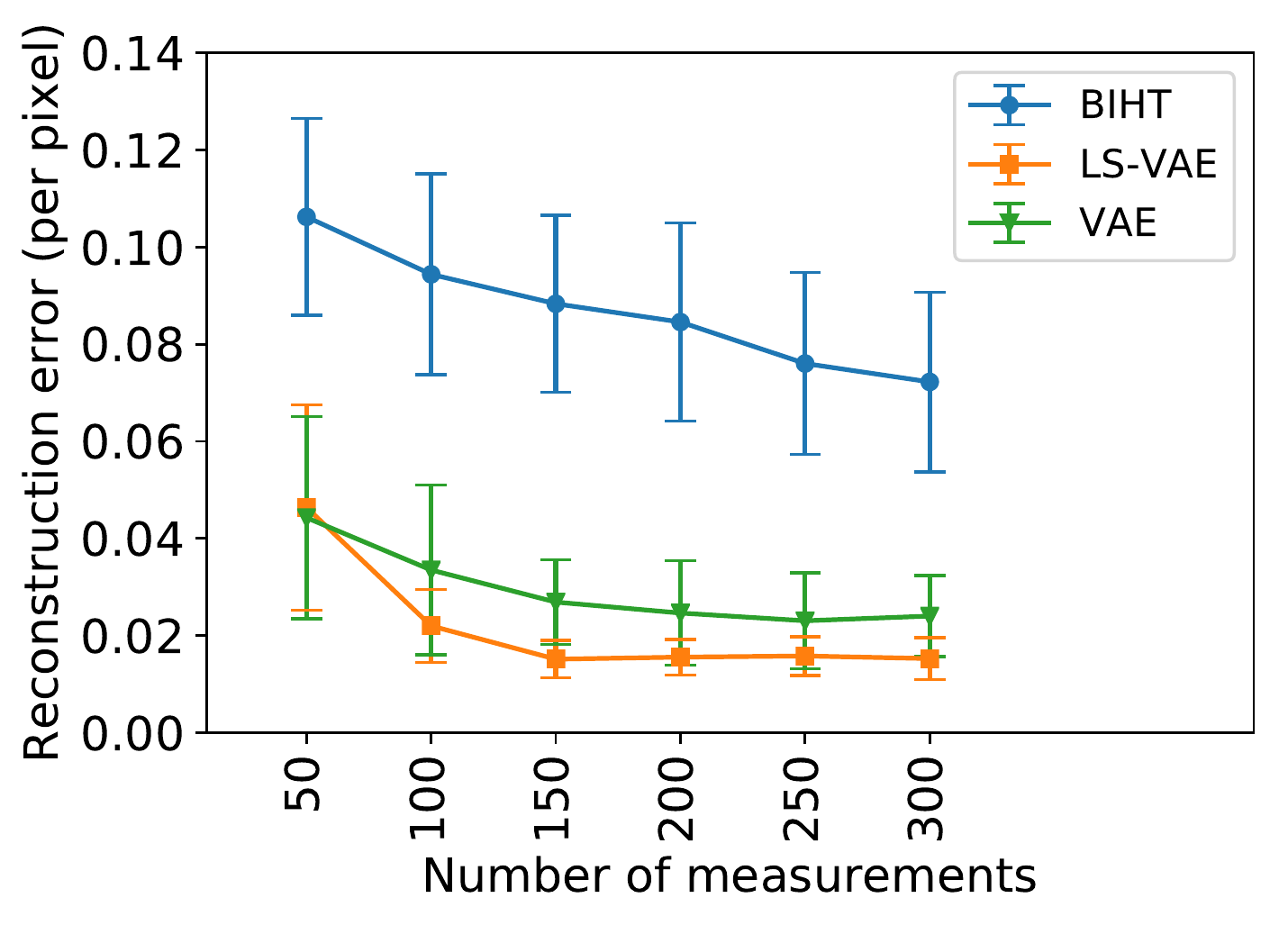}
\includegraphics[scale=0.5]{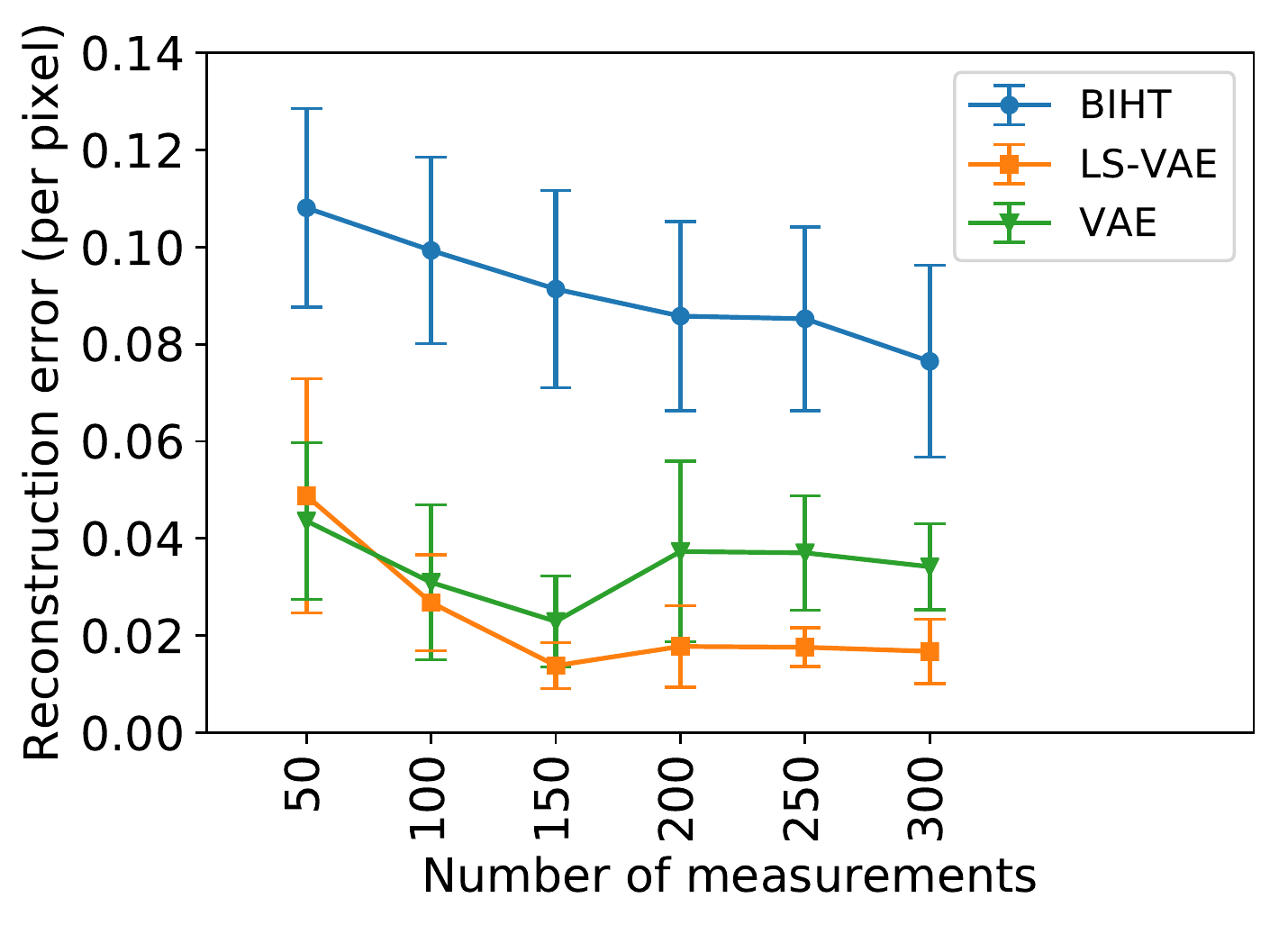}
\setlength{\abovecaptionskip}{1pt}
\caption{pixel-wise reconstruction error as the number of measurements varies. Error bars indicate 95\% confidence intervals.
The  result with  no sign flips and  with 3\% sign flips are shown in the left and right, respectively.}
\end{center}
\end{figure}
\section{Conclusion}
We present a least square decoder by exploring the low generative intrinsic dimension structure of the target for the
1-bit compressive sensing with possible sign-flips.
Under the assumption  that the target signals can be approximately generated via
$L$-Lipschitz generator $G: \mathbb{R}^k\rightarrow\mathbb{R}^{n}, k\ll n$,
we prove that, up to a constant $c$, with high probability, the least square decoder
achieves a sharp  estimation error  $\mathcal{O} (\sqrt{\frac{k\log (Ln)}{m}})$ as long as  $m \geq O( k \log (Ln))$.
We verify  the (approximately) deep generative prior holds if the target signals have low intrinsic dimensions by constructing a ReLU network with  properly chosen depth and width.
Extensive numerical simulations and comparisons with state-of-the-art methods  demonstrated
the least square decoder is the robust to noise and sign flips, which verifies our theory. We only consider the analysis of the  least squares decoders, we will leave the analysis of   the regularized least squares decoder  in the future work.

\section*{Acknowledgements}
Y. Jiao is supported in part
by the National Science Foundation of China under Grant 11871474 and by the research fund of KLATASDSMOE of China.
X. Lu is partially supported by the National Key Research and Development Program of China (No.2018YFC1314600), the National Science Foundation of China (No. 11871385), and the Natural Science Foundation of Hubei Province (No. 2019CFA007).
\bibliographystyle{plain}
\bibliography{csref}

\end{document}